\documentclass[10pt,conference]{IEEEtran} 
\usepackage{graphicx} 
\usepackage{amsmath,amssymb,mathtools}

\usepackage{amsthm}
\usepackage{microtype}
\usepackage{graphicx}
\usepackage{subfigure}
\usepackage{booktabs} 
\usepackage{algpseudocode}
\usepackage{algorithm}
\makeatletter
\let\labelindent\relax
\makeatother
\usepackage{enumitem}

\usepackage{hyperref}
\usepackage{xcolor}

\definecolor{cP2P}{RGB}{0,90,181}       
\definecolor{cAccel}{RGB}{200,50,50}     
\definecolor{cSAC}{RGB}{0,148,115}       
\definecolor{cPPO}{RGB}{190,120,20}      
\theoremstyle{plain}
\newtheorem{theorem}{Theorem}[section]

\theoremstyle{definition}
\newtheorem{definition}[theorem]{Definition}
\newtheorem{assumption}[theorem]{Assumption}
\theoremstyle{remark}

\renewcommand{\Pr}[1]{{\mathbf{Pr}\left[#1\right]}}

\renewcommand{\vec}[1]{\mathbf{#1}}

\title{MARBLE: Multi-Armed Restless Bandits in Latent Markovian Environment}
\author{Mohsen Amiri, Konstantin Avrachenkov, Ibtihal El Mimouni, Sindri Magn\'usson
\thanks{M. Amiri and S. Magn\'usson are with the Department of Computer and Systems Sciences (DSV), Stockholm University, Stockholm, Sweden (e-mail: mohsen.amiri@dsv.su.se; sindri.magnusson@dsv.su.se). K. Avrachenkov is with Inria Sophia Antipolis, Biot, France (e-mail: k.avrachenkov@inria.fr). I. El Mimouni is with Inria Sophia Antipolis, Biot, France, and Smartprofile, Valbonne, France (e-mail: ibtihal.el-mimouni@inria.fr).}
\thanks{This work was partially supported by the Swedish Research Council under grant 2020-03607 and, in part, by Sweden's Innovation Agency (Vinnova).}
}

\date{February 2025}

\begin{document}

\maketitle

\begin{abstract}
Restless Multi-Armed Bandits (RMABs) are powerful models for decision-making under uncertainty, yet classical formulations typically assume fixed dynamics, an assumption often violated in nonstationary environments. We introduce MARBLE (Multi-Armed Restless Bandits in a Latent Markovian Environment), which augments RMABs with a latent Markov state that induces nonstationary behavior. In MARBLE, each arm evolves according to a latent environment state that switches over time, making policy learning substantially more challenging. We further introduce the Markov-Averaged Indexability (MAI) criterion as a relaxed indexability assumption and prove that, despite unobserved regime switches, under the MAI criterion, synchronous Q-learning with Whittle Indices (QWI) converges almost surely to the optimal Q-function and the corresponding Whittle indices. We validate MARBLE on a calibrated simulator-embedded (digital twin) recommender system, where QWI consistently adapts to a shifting latent state and converges to an optimal policy, empirically corroborating our theoretical findings.
\end{abstract}

\section{Introduction}

Resource allocation under uncertainty is a foundational problem across wireless networks, healthcare, and targeted social programs \cite{neely2010stochastic,zhao2022link,burdett2023stochastic,jamil2022resource}. The Restless Multi-Armed Bandit (RMAB) offers a powerful abstraction for these applications so that at each decision epoch, a controller selects $M$ of $N$ arms subject to a budget \cite{whittle1988restless}. Unlike classical bandits, arm states evolve even when passive, making optimal control computationally intractable (PSPACE-hard) \cite{papadimitriou1999complexity,akbarzadeh2022conditions}. A seminal advance by Whittle \cite{whittle1988restless} proposed an index policy that ranks arms by a problem-dependent priority; activating the top $M$ indices yields a scalable heuristic that is asymptotically optimal in various regimes \cite{weber1990index}.

In recent years, model-free methods for learning the Whittle index have experienced rapid growth. Early work on learning index policies includes \cite{duff1995q} for the Multi-Armed Bandit (MAB) problem for learning Gittins indices \cite{gittins2011multi}. For Whittle indices, the study in \cite{fu2019towards} proposed a non-convergent method. In contrast, the contribution in \cite{gibson2021novel} iteratively learns a Whittle-like policy, but it is tabular and does not scale. Foundational results encompass indexability for dynamic channel access \cite{liu2010indexability}, applications to stochastic scheduling \cite{hsu2018age}, and threshold indexability with practical heuristics \cite{nino2020verification}. Under a time-average criterion, the study in \cite{avrachenkov2022whittle} provided a convergent tabular algorithm, which was extended to Neural Networks (NN) in \cite{pagare2023full} and further analyzed for finite time in \cite{xiong2023finite}. Other directions include a Q-learning Lagrangian method for multi-action RMABs \cite{killian2021q} and the NN-based NeurWIN algorithm \cite{nakhleh2021neurwin}, which maximizes discounted returns without guaranteeing convergence to true Whittle indices. Recent model-free methods estimate Whittle indices from data using two-timescale stochastic approximation, thus avoiding explicit system identification \cite{avrachenkov2022whittle,robledo2024tabular}.

Classical formulations typically assume stationary dynamics, an assumption that often fails in practice \cite{li2022efficient}. This motivates learning-based RMAB approaches that allow initially changing dynamics \cite{besbes2014stochastic,shisher2025online,gafni2022restless}. Beyond stationarity, prior work includes nonstationary bandits under a variation-budget model with optimal regret \cite{li2022efficient,besbes2014stochastic}, globally modulated RMABs with regime-aware indices \cite{gafni2022restless}, and Whittle-style online index learning, using windowing or change detection, with provable dynamic-regret guarantees \cite{shisher2025online}.

Existing studies on changing dynamics often assume prior knowledge of when and how the dynamics change or that the relevant context is observable. In RL, these assumptions frequently fail: the dynamics may be hidden and unrecoverable from observations \cite{amiri2025reinforcement,amiri2024convergence}. In such cases, neither the context nor the environment's dynamics can be learned, raising the question of whether classical learning-based RL methods still converge. In \cite{amiri2024convergence}, the RL problem considers a scenario where the reward function varies at each iteration while the transition probabilities remain stationary. By contrast, \cite{amiri2025reinforcement} examines a more general setting in which the entire dynamics change abruptly, governed by an unobserved Markovian latent state. Therefore, the same question remains open for learning-based RMAB methods under nonstationarity driven by an unobserved Markovian latent state, which is a gap this work addresses.

We introduce MARBLE (Multi-Armed Restless Bandit in a Latent Markovian Environment), a framework that models nonstationary dynamics with abrupt regime changes via a latent Markov state. Furthermore, we introduced Markov-Average Indexability (MAI) as a relaxed indexability assumption and studied the synchronous Q-learning \cite{avrachenkov2022whittle, borkar2008stochastic,lakshminarayanan2017stability} with Whittle Index (QWI) under the MARBLE framework. Besides, we provided theoretical guarantees showing that under the MAI criterion, synchronous QWI converges to the optimal Q-table and, in turn, to the optimal Whittle indices, thereby attaining the optimal policy \cite{avrachenkov2022whittle}. Finally, we validate MARBLE on a recommender-system task, in which the platform maintains a calibrated simulator (digital twin) enabling full-sweep, planning-style synchronous updates, demonstrating empirically that synchronous QWI converges to an optimal policy in this setting. 

\section{Preliminaries}\label{sec:preliminaries}

\subsection{The restless multi–armed bandit}\label{subsec:rb}
A Restless bandit extends the classical bandit problem~\cite{lattimore2020bandit} in two key ways: all arms evolve according to Markovian dynamics regardless of selection; and budget constraints limit the number of arms activated per timestep~\cite{whittle1988restless,guha2010approximation}. Let $\mathcal N=\{1,\dots,N\}$ index $N$ independent arms. Arm \(i\) has finite states \(\mathcal{S}\) and actions \(\mathcal{A}=\{0,1\}\) (1 active, 0 passive); at time \(k\in\mathbb{N}\), if \(S_k^i=s\) and \(A_k^i=a\), it yields reward \(\vec r^{\,i}(s,a)\) and transitions to \(s'\) with probability \(p^{\,i}(s'\mid s,a)=\Pr{S_{k+1}^i=s'\mid S_k^i=s, A_k^i=a}\). An instantaneous budget enforces that exactly $M$ arms are active each period, i.e., $\sum_{i=1}^{N} A_k^i= M,$ where $1\le M<N$. For a discount factor $\gamma\in(0,1)$, the control objective is to find a policy $\pi$ maximizing
\begin{align}
\max_{\pi}\;
\mathbb E\!\left[\sum_{i=1}^N\sum_{k=0}^{\infty} \gamma^k R^i_k\right],
\label{Equ:RB-obj}
\end{align}
where $R^i_k=\vec r^{i}\!\left(S_k^i,A_k^i \right)$ . However, Whittle's method \cite{whittle1988restless,liu2010indexability} modifies the constraint so that it is only enforced in expectation. Specifically, the constraint is relaxed to $\mathbb{E}\!\Bigl[\sum_{i=1}^N\sum_{k=0}^{\infty}\gamma^k A_{k}^{i}\Bigr] \;\le\; \frac{M}{1-\gamma}$. Introducing a Lagrangian multiplier $\lambda$, we convert the constrained problem in Eq.~\eqref{Equ:RB-obj} into the following unconstrained formulation:
\begin{align}\label{Equ:RB-obj-relaxed} 
\max_{\pi}\;\mathbb E\left[\sum_{i=1}^N\sum_{k=0}^{\infty} \gamma^k \hat R^i_k\right].
\end{align}
where $\hat R^i_k = R^i_k + \lambda \left(1 -  A_{k}^{i}\right)$. Hence, the Q-function under any stationary policy $\pi:\mathcal S\!\to\!\mathcal A$ and Lagrangian multiplier $\lambda$ are defined as follows:
\begin{align}\label{eq:value-Q-function}
&\vec Q(s,a)=\mathbb E\!\left[\sum_{i=1}^N \sum_{k=0}^{\infty}\gamma^k \hat R^i_k\middle|X^i_0=(s,a) \right],  
\end{align}
where $X^i_k = (S^i_k, A_k^i)$ . Due to Whittle's relaxation, specifically the relaxed constraint, the Q-function decomposes into a superposition of the per-arm Q-functions for each arm $i$ as $\vec Q(s,a)= \sum_{i=1}^N \vec Q^i_{\lambda}(s,a)$, where $\vec Q^i_{\lambda}(s,a)$ is Q-function of each arm $i$ that are defined as follows:
\begin{align}\label{eq:value-Q-function-arm-i}
&\vec Q^i_{\lambda}(s,a)=\mathbb E\!\left[\sum_{k=0}^{\infty}\gamma^k \hat R^i_k\middle|X^i_0=(s,a) \right],  
\end{align}
Under the relaxed problem in Eq.~\eqref{Equ:RB-obj-relaxed}, the decoupling of the Q-function implies that each arm $i$ can be analyzed via a single--arm MDP with a subsidy for passivity $\lambda$. For a given $\lambda$, define the passive region as $\mathcal P^i(\lambda)=\{s\in\mathcal S:\vec Q^{i}_\lambda(s,0) \ge \vec Q^{i}_\lambda(s,1)\}$.
 
\begin{definition}[Indexability]
Arm $i$ is \emph{indexable} if $\lambda\mapsto\mathcal P^i(\lambda)$ is monotonically increasing (by inclusion), with $\mathcal P^i(\lambda)=\varnothing$ for $\lambda\ll 0$ and $\mathcal P^i(\lambda)=\mathcal S$ for $\lambda\gg 0$.
\end{definition}
 
When arm \(i\) is indexable, the Whittle index \(\lambda^{i}_{\star}(s)\) at state \(s\) is defined as the smallest subsidy that makes passivity optimal at \(s\): \(\lambda^{i}_{\star}(s):=\inf\{\lambda\in\mathbb{R}:\; s\in\mathcal{P}^{i}(\lambda)\}\). The index $\lambda^i_\star(s)$ quantifies the urgency of activating arm $i$ in state $s$: the larger the index, the larger the subsidy one would need in order to justify keeping the arm passive in that state. Given a system state $(S_k^1,\dots,S_k^N)$ at time $k$, the Whittle index policy activates the $M$ arms with the largest indices $\lambda^i_\star(S_k^i)$.
 
To compute the Whittle index at a reference state $z \in \mathcal{S}$, we fix the Lagrangian multiplier $\lambda^i(z)$ for arm $i$ and write the corresponding Bellman equation:
\begin{align}\label{eq:value-Q-function-arm-i-relation}
&\vec Q^{i,z}_{\lambda}(s,a)
= \vec r^{i}\!\left(s,a \right) + \lambda^i(z) (1-a) \;\notag \\
&+\; \gamma \sum_{s' \in \mathcal{S}} p^{i}(s'|s,a)\max_{a'\in\mathcal A} \vec Q^{i,z}_{\lambda}(s',a').
\end{align}
The Whittle index $\lambda^i_\star(z)$ is then equivalently the value of $\lambda^i(z)$ that makes the two actions indifferent at the reference state, i.e., $\vec Q^{i,z}_{\lambda_\star}(z,1) = \vec Q^{i,z}_{\lambda_\star}(z,0)$.

\subsection{Synchronous Q–learning with Whittle index (QWI)}\label{subsec:QWI} Under indexability, QWI estimates each arm’s Whittle index $\lambda_\star^i(\cdot)$ by learning the subsidy that equalizes actions \cite{robledo2024tabular}. We keep a state map $\lambda:\mathcal S\to\mathbb R$ and per-arm tabular values $\vec Q_k^{\,i}$ for the relaxed single-arm return (Eq.~\eqref{Equ:RB-obj-relaxed}).

On the fast timescale, for each arm $i$ at reference state $z$ with current estimate $\lambda_k^i$, and using the calibrated simulator $\mathcal{G}_i(\theta)$ (treating $\theta$ as known), for the trajectory $(S_k^i, A_k^i, R_k^i, S_{k+1}^i)$ we have:
\begin{align}\label{Equ:QWI-QL-update}
&\vec Q_{k+1}^{i,z}(S^i_k,A^i_k)
=(1-\alpha_k)\vec Q_{k}^{i,z}(S^i_k,A^i_k) \nonumber\\
&\quad+\alpha_k\!\left(R_k^i+\lambda_k^i(z)(1-A^i_k)+\gamma\max_{a\in\{0,1\}}\vec Q_{k}^{i,z}(S^i_{k+1},a)\right),
\end{align}
and subsequently, for the slow timescale, we have:
\begin{align}\label{Equ:QWI-index-update}
\lambda^{i}_{k+1}(z)=\lambda^{i}_k(z)+\beta_k\!\left(\vec Q_{k+1}^{i,z}(z,1)-\vec Q_{k+1}^{i,z}(z,0)\right),
\end{align}
which is driving the action gap to $0$.  We assume generative access to a calibrated simulator (digital twin) $\mathcal{G}_i(\theta)$ that returns the environment–averaged transition kernel $p_{\mathcal E}^i(s'\mid s,a;\theta)$ ($\theta$ is known). At each synchronous iteration $k$, and for every state–action pair $(s,a)$, we sample a Monte-Carlo next state $s' \sim p_{\mathcal E}^i(\cdot\mid s,a;\theta)$. The environment then emits a one–step reward $R_k^i \;=\; r_{E_k}^i(s,a)$, where $E_k$ is the latent (unobserved) environment state at time $k$; we assume the environment shows $R_k^i$ (but not $E_k$) for all state-action pairs.


\begin{assumption}\label{assump:stepsizes-timescale}
Let $\{\alpha_k\}$ and $\{\beta_k\}$ be deterministic stepsize sequences with $\sum_{k=0}^\infty \alpha_k=\sum_{k=0}^\infty \beta_k=\infty$, $\sum_{k=0}^\infty \alpha_k^2<\infty$, $\sum_{k=0}^\infty \beta_k^2<\infty$, and $\beta_k=o(\alpha_k)$ as $k\to\infty$.
\end{assumption}

Algorithm~\ref{Alg:QWI} represents a practical tabular implementation of synchronous QWI. At each outer iteration $k$, assuming generative access to a calibrated simulator $\mathcal{G}_i(\theta)$, for all simulated state-action pair we use Eq.~\eqref{Equ:QWI-QL-update} for every arm $i$. On the slower timescale, update the index estimate at the reference state $z$ according to Eq.~\eqref{Equ:QWI-index-update}. For online execution at time $k$, observe all arm states $S^i_k$, form index estimates $\lambda^i_k$, and select a joint action $A^i_k$ that activates exactly $M$ arms: with probability $1-\varepsilon$, activate the $M$ arms with largest $\lambda^i_k$; with probability $\varepsilon$, choose $M$ arms uniformly at random.



\begin{algorithm}[h]
\caption{Tabular Synchronous QWI}\label{Alg:QWI}
\begin{algorithmic}[1]
\Require  $M$, $\gamma$, $\alpha_k$, $\beta_k$, $\mathcal{G}_i(\theta)$, $\varepsilon$
\State Initialize the Q-tables and Whittle indices.
\For{$k=0,1,2,\ldots$}
  \For{each arm $i \in \mathcal{N}$} 
    \State Update $\vec Q_k^{\,i,z}$ for all $(s,a,s') \sim  \mathcal{G}_i(\theta)$ using \eqref{Equ:QWI-QL-update} 
    \State Update $\lambda^{\,i}(z)$ using \eqref{Equ:QWI-index-update}.
  \EndFor
\EndFor
\end{algorithmic}
\end{algorithm}

\section{MARBLE: Model and Lagrangian Decomposition}\label{sec:marble}

We now introduce the MARBLE model, which augments RMAB with an unobservable environment that modulates rewards and transitions. Let the latent environment state \(E_k\) take values in a finite set \(\mathcal{E}\) and evolve as a Markov chain with transition matrix \(H(e'\mid e)=\Pr{E_{k+1}=e'\mid E_k=e}\), where both the environment states \(E_k\) and \(H(\cdot)\) are unobservable.

\begin{assumption}\label{assump:A1}
The latent chain $\{E_k\}$ is irreducible and aperiodic with unique stationary distribution $\vec\mu_{\mathcal E}$. 
\end{assumption}

\begin{definition}[MARBLE]\label{Def:MARBLE}
A MARBLE instance is the tuple \(\bigl(\mathcal N,\mathcal S,\{\vec r_e^{\,i}\}_{i \in \mathcal{N},e \in \mathcal{E}},\{p_e^{\,i}\}_{i \in \mathcal{N},e \in \mathcal{E}},M;\mathcal E,H\bigr)\), where, for each arm \(i\in\mathcal N\) and environment state \(e\in\mathcal E\), \(p_e^{\,i}(s'\mid s,a)=\Pr{S_{k+1}^i=s'\mid S_k^i=s,\,A_k^i=a,\,E_k=e}\) and \(\vec r_e^{\,i}(s,a)=\vec r^{\,i}(s,a,e)\).
\end{definition}

In this setting, since  the  Q-function depends on the environment state, for a stationary policy $\pi$ and discount factor $\gamma \in (0,1)$ we define

\begin{align}\label{Equ:MARBLE-V-Q-F-raw}
\vec{Q}(s,a;e)
&= \mathbb{E}\!\left[\sum_{k=0}^{\infty}\sum_{i=1}^N\gamma^k R_k^i\;\middle|\; Y_0=(s,e,a)\right]. 
\end{align}
where $Y_k=(S^i_k,E_k,A^i_k)$. By relaxing the instantaneous budget, the Q-function in Eq.~\eqref{Equ:MARBLE-V-Q-F-raw} yield the Lagrangian with passivity subsidy $\lambda$:

\begin{align}\label{Equ:MARBLE-V-Q-F-raw-Const}
\vec{Q}(s,a;e)
&= \mathbb{E}\!\left[\sum_{k=0}^{\infty}\sum_{i=1}^N\gamma^k\hat R^i_k\;\middle|\; Y_0=(s,e,a)\right].
\end{align}
Thus, the Q-function decomposes as a sum over arms, with each term given by the per-arm Q-functions for arm \(i\): \(\vec Q(s,a;e)=\sum_{i=1}^{N}\vec Q^{i}_{\lambda}(s,a;e)\), where \(\vec Q^{i}_{\lambda}(s,a;e)\) denotes the Q-function for arm \(i\), defined as follows:

\begin{align}\label{Equ:MARBLE-V-Q-F-raw-Const-arm-i}
&\vec Q^i_{\lambda}(s,a;e)=\mathbb E\!\left[\sum_{k=0}^{\infty}\gamma^k \hat R^i_k\middle|Y_0=(s,e,a) \right]. 
\end{align}
Therefore, similar to Eq.~\eqref{eq:value-Q-function-arm-i-relation}, for reference state $z \in \mathcal{S}$ and reference environmental state $u \in \mathcal{E}$, we have ($z$ and $u$ are fixed.):
\begin{align}\label{Equ:MARBLE-V-Q-F-raw-relation}
&\vec{Q}^{i,z,u}_{\lambda}(s,a;e)
=  \vec{r}_{e}^{i}(s,a)+\lambda^i(z;u) (1-a) \notag\\
&+\gamma \sum_{s'\in\mathcal S}p_{e}^{i}(s'|s,a)\,\max_{a' \in \mathcal{A}}\vec{Q}^{i,z}_{\lambda}(s',a';e).
\end{align}
so, the Whittle index in this situation is called $\lambda_\star^i(z;u)$ so that $\vec{Q}^{i,z,u}_{\lambda_\star}(z,0;u)=\vec{Q}^{i,z,u}_{\lambda_\star}(z,1;u)$.


\begin{assumption}[Boundedness]\label{assump:bounded}
There exist finite constants \(R_{\max}\) such that, for all \(i\in\mathcal{N}\), \(s\in\mathcal{S}\), \(a\in\{0,1\}\), and \(e\in\mathcal{E}\), we have \(\|\vec r_{e}^{\,i}(s,a)\|_{\infty}\le R_{\max}\).
\end{assumption}

\begin{assumption}[Stationary Sampling]\label{assump:stationary-sampling}
The initial environment state $\tilde{E}$ and reference environment state $\tilde{U}$
are drawn from the stationary distribution $\vec\mu_{\mathcal{E}}$ (guaranteed to exist
by Assumption~\ref{assump:A1}).
\end{assumption}

Under Assumption~\ref{assump:stationary-sampling}, we define the environment-agnostic
(average) $Q$-function as
$\vec{\bar{Q}}^{i,z}_{\bar{\lambda}}(s,a)
:=\mathbb{E}_{\tilde{E},\tilde{U}\sim\vec\mu_{\mathcal{E}}}
\!\left[\vec{Q}_{\lambda}^{i,z,\tilde{U}}(s,a,\tilde{E})\right]$. Therefore, the Bellman equation for $\vec{\bar{Q}}^{i,z}_{\bar{\lambda}}(s,a)$ is as follows: 
\begin{align}\label{eq:MARBLE-bellman-opt}
&\vec{\bar{Q}}^{i,z}_{\bar{\lambda}}(s,a)
=  \vec{r_{\mathcal{E}}}^{i}(s,a)+\bar{\lambda}^i(z) (1-a)\notag\\
&+\gamma \sum_{s'\in\mathcal S} p_{\mathcal{E}}^{i}(s'|s,a)\,\max_{a' \in \mathcal{A}}\vec{\bar{Q}}^{i,z}_{\bar{\lambda}}(s',a').
\end{align}
where we define average reward as $\vec{r}_{\mathcal{E}}^{i}(s,a)=\sum_{e\in\mathcal E}\vec\mu_{\mathcal{E}}(e)\vec{r}_e^{i}(s,a)$ and average transition probability as $p_\mathcal{E}^{i}(s'\mid s,a)=\sum_{e\in\mathcal E}\vec\mu_{\mathcal{E}}(e)p_e^{i}(s'\mid s,a)$. Thus, the environment-agnostic Whittle index $\bar{\lambda}_\star^i(z)$ is the value of $\bar{\lambda}^i(z)$ that makes the two actions indifferent, i.e., $\vec{\bar{Q}}^{i,z}_{\bar{\lambda}_\star}(z,0)=\vec{\bar{Q}}^{i,z}_{\bar{\lambda}_\star}(z,1)$. Hence, at the Whittle index $\bar{\lambda}_\star^{i}(\cdot)$ for arm $i$, the optimal policy is called $\pi_\star(\cdot)$.

\begin{assumption}(Markov-Average Indexibility (MAI))\label{assump:indexability-MARBLE}
For each arm $i$, the environment-averaged single-arm subsidy problem is indexable: the passive region 
$\bar{\mathcal P}^{\,i}(\lambda):=\{z:\ \vec{\bar Q}^{i}_\lambda(z,0)\ge \vec{\bar Q}^{i}_\lambda(z,1)\}$ 
is increasing in $\lambda$, expanding from $\varnothing$ to $\mathcal S$ as $\lambda$ increases. 
\end{assumption}

Assumption MAI relaxes the standard indexability requirement: instead of assuming every environment is indexable, it requires only that the average environment be indexable.

\section{Algorithm and Main Results}\label{sec:algo-main}

In the MARBLE model, each arm’s dynamics are driven by a latent environmental state evolving as a Markov chain, introducing nonstationarity that could, a priori, hinder convergence of the synchronous QWI to the optimal Q-function and Whittle indices. However, Theorem~\ref{Thm:MARBLE-QWI} establishes that under the MAI criterion, and given calibrated simulators $\mathcal{G}_i(\theta)$ for all arms $i \in \mathcal{N}$ with known $\theta$, the synchronous QWI does converge to the optimal Q-function, and the associated Whittle indices converge to their optimal values.


\begin{theorem}\label{Thm:MARBLE-QWI}
Assume Assumptions~\ref{assump:stepsizes-timescale}, ~\ref{assump:bounded}, and \ref{assump:indexability-MARBLE} (MAI) hold, and for all arms $i \in \mathcal{N}$ the calibrated simulators $\mathcal{G}_i(\theta)$ with known $\theta$ are available. Then, under the MARBLE model, QWI in Algorithm~\ref{Alg:QWI} converges almost surely: for every arm \(i\), state \(s,z \in \mathcal{S}\), and action \(a \in \{0,1\}\), \(\vec Q_{k}^{i,z}(z,a) \to \vec{\bar{Q}}^{i,z}_{\bar{\lambda}_\star}(z,a)\) and \(\lambda^i_k(z) \to \bar{\lambda}^i_\star(z)\), where \(\bar{\lambda}^i_\star(z)\) is the root of \(\vec{ \bar{Q}}^{i}_{\bar{\lambda}_\star}(z,1)-\vec{ \bar{Q}}^{i}_{\bar{\lambda}_\star}(z,0)=0\), and \(\vec{ \bar{Q}}^{i}_{\bar{\lambda}_\star}\) is the fixed point of the single-arm Bellman operator in Eq.~\eqref{eq:MARBLE-bellman-opt}.
\end{theorem}

\noindent\textit{Proof sketch.}
We cast QWI into Borkar's two-timescale SA framework~\cite{borkar2008stochastic}. 
Independence of the latent state $E_k$ from the arm-state filtration ensures the 
noise is a zero-mean martingale difference. The environment-averaged Bellman operator 
is a $\gamma$-contraction, giving a unique fast-ODE attractor; the MAI criterion 
makes the action gap strictly decreasing, yielding a unique slow-ODE equilibrium at 
$\bar\lambda_\star^{i}(z)$. Boundedness follows from ODE stability and the 
Robbins--Siegmund lemma.

\begin{proof} To establish the convergence of QWI, we invoke Borkar's two-timescale Stochastic Approximation (SA) theorem \cite{borkar1997stochastic, borkar2008stochastic}. For clarity and to facilitate the proof, we restate the theorem below.

\begin{theorem}[Borkar’s Two–timescale SA Theorem {\cite{borkar1997stochastic,borkar2008stochastic}}]
\label{thm:Borkar-2TSA}
Let $\{x_k\}_{k\ge 0}\subset\mathbb{R}^{d_1}$ and
$\{y_k\}_{k\ge 0}\subset\mathbb{R}^{d_2}$ satisfy the coupled recursions
\begin{align}
x_{k+1} &= x_k \;+\; a_k\!\left[\,h\!\left(x_k,y_k\right) + M_{k+1}\right], \label{eq:fast-SA}\\
y_{k+1} &= y_k \;+\; b_k\!\left[\,g\!\left(x_k,y_k\right) + N_{k+1}\right], \label{eq:slow-SA}
\end{align}
where
\begin{enumerate}[label=\bfseries(A$_\arabic*$)]
  \item\label{A1} $h:\mathbb{R}^{d_1}\!\times\!\mathbb{R}^{d_2}\to\mathbb{R}^{d_1}$ and
        $g:\mathbb{R}^{d_1}\!\times\!\mathbb{R}^{d_2}\to\mathbb{R}^{d_2}$ are
        globally Lipschitz;
  \item\label{A2} $\bigl\{\mathcal{F}_k\bigr\}_{k\ge0}$ is the canonical filtration
        with
        \[
           \mathcal{F}_k \;=\;
           \sigma\!\Bigl(x_0,y_0, \cdots, x_k, y_k,
               M_{1},N_{1},\dots,
               M_{k},N_{k}\Bigr),
        \]
        so that $(x_k,y_k)$ is $\mathcal{F}_k$–measurable for each $k$;
  \item\label{A3} $\{M_{k}\}$ and $\{N_{k}\}$ are \emph{martingale–difference} sequences
        w.r.t.\ $\{\mathcal{F}_k\}$, i.e.\
        $\mathbb{E}[M_{k+1}\mid\mathcal{F}_k]=0$, $\mathbb{E}[N_{k+1}\mid\mathcal{F}_k]=0$, and $(\exists\,C<\infty)$
        \begin{align*}
          \mathbb{E}\!\bigl[\|M_{k+1}\|^2\mid\mathcal{F}_k\bigr] & \;\le\;
          C\bigl(1+\|x_k\|^2+\|y_k\|^2\bigr), \\
          \mathbb{E}\!\bigl[\|N_{k+1}\|^2\mid\mathcal{F}_k\bigr]&\;\le\;
          C\bigl(1+\|x_k\|^2+\|y_k\|^2\bigr)
        \end{align*}
  \item\label{A4} the step–sizes follow $\sum_k a_k=\sum_k b_k=\infty$, $\sum_k a_k^{2}<\infty$, $\sum_k b_k^{2}<\infty$, and $\displaystyle\frac{b_k}{a_k}\to 0$ as $k\to\infty$.
\end{enumerate}

\noindent
Assume moreover:
\begin{enumerate}[label=\bfseries(B$_\arabic*$)]
\item\label{B1} \textbf{Fast‐ODE attractor:}
      For every fixed $y\in\mathbb{R}^{d_2}$ the ODE
      $\dot x = h(x,y)$ has a \emph{unique globally asymptotically stable}
      equilibrium $x^{\star}(y)$, and the map
      $y\mapsto x^{\star}(y)$ is Lipschitz.
\item\label{B2} \textbf{Slow‐ODE attractor:}
      Define $\bar g(y):=g\!\bigl(x^{\star}(y),y\bigr)$.
      The ODE $\dot y = \bar g(y)$ has a \emph{unique globally
      asymptotically stable} equilibrium $y^{\star}$.
\item\label{B3} \textbf{Boundedness:}
      The iterates are a.s.\ bounded, i.e.\
      $\sup_{k}\bigl(\|x_k\|+\|y_k\|\bigr)<\infty$ almost surely.
\end{enumerate}

\noindent
Then, with probability~$1$, $x_k\to x^{\star}(y^{\star})$ and $y_k\to y^{\star}$ as $k\to\infty$; hence $(x_k,y_k)\to\bigl(x^{\star}(y^{\star}),\,y^{\star}\bigr)$.

\end{theorem}

To proceed with the proof, we first recast the update rules in Eq.~\eqref{Equ:QWI-QL-update} and Eq.~\eqref{Equ:QWI-index-update} into the standard forms of Equations~\eqref{eq:fast-SA} and~\eqref{eq:slow-SA}, respectively. Specifically, for arm \( i \), given that the environment state is \( E_k \) and the reference state $z$, for all $(s,a,s') \sim \mathcal{G}_i(\theta)$ and the trajectory \( (S_k=s, A_k=a, (r_k^j)_{j \in \{0,1\}}, S_{k+1}=s') \), we define \( \Phi^i(\cdot) \) and \( M^i_{k+1} \) as follows:

\begin{align}\label{eq:Phi-M}
&(T_{\lambda_k} \vec{Q}_{k}^{i,z})(s,a) =(1-a) \left( \vec r^i_{\mathcal{E}}(s,0) + \lambda_k^i(z)  \right) + a \vec r^i_{\mathcal{E}}(s,1) \notag \\
&+ \gamma\sum_{s' \in \mathcal{S}}p_{\mathcal{E}}^{i}(s'|s,a) ~\underset{a' \in \mathcal{A}} {\max}~\vec{Q}_{k}^{i,z}(s', a') \notag\\
&M^{i,z}_{k+1}(s,a) = (1-a) \left( r_k^0-\vec r^i_{\mathcal{E}}(s,0)  \right) + a \left( r_k^1-\vec r^i_{\mathcal{E}}(s,1)\right) \notag\\
&+ \gamma \, \underset{a' \in \mathcal{A}} {\max}~\vec{Q}_{k}^{i,z}(s', a') - \gamma\sum_{s' \in \mathcal{S}}p_{\mathcal{E}}^{i}(s'|s,a) ~\underset{a' \in \mathcal{A}} {\max}~\vec{Q}_{k}^{i,z}(s', a')
\end{align} 
where $r_k^0=\vec{r}_{E_k}^{i}(s,0)$ and $r_k^1=\vec{r}_{E_k}^{i}(s,1)$. Therefore, we can rewrite Eq.~\eqref{Equ:QWI-QL-update} as:

\begin{align}\label{Equ:MARBLE-QL-update_rule_reform}
\vec{Q}_{k+1}^{i,z}(s, a) &= \vec{Q}_{k}^{i,z}(s, a) + \alpha_k \bigl((T_{\lambda_k} \vec{Q}_{k}^{i,z})(s,a) - \vec{Q}_{k}^{i,z}(s, a) \notag \\
&+ M^{i,z}_{k+1}(s,a)\bigr)
\end{align}

By comparing Eq.~\eqref{eq:fast-SA} and Eq.~\eqref{Equ:MARBLE-QL-update_rule_reform}, it is easy to verify that $a_k=\alpha_k$, $h(x_k, y_k)=(T_{\lambda_k} \vec{Q}_{k}^{i,z})(\cdot) - \vec{Q}_{k}^{i,z}(\cdot)$, where $x_k=\vec{Q}_{k}^{i,z}(\cdot)$ and $y_k=\lambda^i_{k}(z)$ are the Q-function and Whittle index estimate respectively. On the other hand, by comparing Eq.~\eqref{eq:slow-SA} and Eq.~\eqref{Equ:QWI-index-update}, we have $b_k=\beta_k$, $g(x_k, y_k)=\vec{Q}_k^{i,z}(s, 1)-\vec{Q}_k^{ i,z}(s, 0)$, and the martingale difference sequence $N_{k+1}=0$.

We first show that assumption~\ref{A1} is established. By the definition of $h(x_k, y_k)$ we have,
  \begin{align*}
      &\|h(x_k, y_k) - h(x_{k'}, y_{k'}) \|_{\infty} \\
      &= \| (T_{\lambda_k} \vec{Q}_{k}^{i,z} - T_{\lambda_{k'}} \vec{Q}_{k'}^{i,z})(s,a) - (\vec{Q}_{k}^{ i,z} - \vec{Q}_{k'}^{i,z})(s, a) \|_{\infty} \\
        & \leq \| (T_{\lambda_k} \vec{Q}_{k}^{i,z} - T_{\lambda_{k'}} \vec{Q}_{k'}^{i,z})(s,a) \|_{\infty} \\
        &+ \| (\vec{Q}_{k}^{ i,z} - \vec{Q}_{k'}^{i,z})(s, a) \|_{\infty}\\
        &\leq \|\gamma \sum_{s' \in \mathcal{S}}p_{\mathcal{E}}^{i}(s'|s,a) \underset{a' \in \mathcal{A}}{\max} ~\left( x_k -  x_k'\right)(s',a') \|_\infty  \\
        &+ (1-a) \|y_k-y_{k'}\|_{\infty} + \|x_k-x_{k'}\|_{\infty} \\
        &\leq (1-a) \|y_k-y_{k'}\|_{\infty} +(1+\gamma)\|x_k-x_{k'}\|_{\infty}
   \end{align*}
Therefore, since $\gamma \in [0,1)$ and $a \in \{0,1\}$, we can conclude that $h(x_k, y_k)$ is globally Lipschitz.

For proving the assumption \ref{A3}, we begin by showing that the sequence $\{M^{i,z}_{k+1}\}$ is a martingale‐difference sequence with uniformly bounded second moments. Then, we have,
\begin{align*}
&\mathbb{E}\bigl[M^{i,z}_{k+1}(s,a)\bigr] = (1-a) \left( \mathbb{E}\bigl[r_k^0\bigr]-\vec r^i_{\mathcal{E}}(s,0)  \right) \notag \\
&+ a \left( \mathbb{E}\bigl[r_k^1\bigr]-\vec r^i_{\mathcal{E}}(s,1)\right) + \gamma \, \mathbb{E}\bigl[\underset{a' \in \mathcal{A}} {\max}~\vec{Q}_{k}^{i,z}(s', a')\bigr] \notag \\
&- \gamma\sum_{s' \in \mathcal{S}}p_{\mathcal{E}}^{i}(s'|s,a) ~\underset{a' \in \mathcal{A}} {\max}~\vec{Q}_{k}^{i,z}(s', a')
\end{align*}
For arm $i$, $\mathbb{E}[r_k^j]=\sum_{e\in\mathcal{E}}\Pr{E_k=e\mid\mathcal{F}_k}\,\vec r_e(s,j)$ for $j\in\{0,1\}$, and $\mathbb{E}\!\left[\max_{a'\in\mathcal{A}}\vec Q_{k}^{i,z}(s',a')\right]=\sum_{s'\in\mathcal{S}}\Pr{S^i_k=s'\mid\mathcal{F}_k}\,\max_{a'\in\mathcal{A}}\vec Q_{k}^{i,z}(s',a')$. Since $S^i_k=s$ and $A^i_k=a$ are $\mathcal{F}_k$-measurable while $E_k$ is not, and by the MARBLE property, $\Pr{E_k=e\mid\mathcal{F}_k}=\Pr{E_k=e}=\mu_{\mathcal{E}}(e)$ and $\Pr{S^i_k=s'\mid\mathcal{F}_k}=\Pr{S^i_{k+1}=s'\mid S^i_k=s, A^i_k=a}$, which is $p_{\mathcal{E}}^{\,i}(s'\mid s,a)$. Hence, $\mathbb{E}[r_k^j]=\vec r^i_{\mathcal{E}}(s,j)$ and $\mathbb{E}\bigl[\underset{a' \in \mathcal{A}} {\max}~\vec{Q}_{k}^{i,z}(s', a')\bigr]=\sum_{s' \in \mathcal{S}}p_{\mathcal{E}}^{i}(s'|s,a) ~\underset{a' \in \mathcal{A}} {\max}~\vec{Q}_{k}^{i,z}(s', a')$, which they lead to $\mathbb{E}\bigl[M^{i,z}_{k+1}(s,a)\bigr]=0$.

We can write $M^{i,z}_{k+1}=U_k+V_k$ with
\begin{align*}
U_k:&=(1-A^i_k)\big(\vec r^{\,i}_{E_k}(S^i_k,0)-r^{i}_{\mathcal{E}}(S^i_k,0)\big) \\
   &+A^i_k\big(\vec r^{\,i}_{E_k}(S^i_k,1)-r^{\,i}_{\mathcal{E}}(S^i_k,1)\big), \\
V_k:&=\gamma\Big(\max_{a'}\vec Q^{i,z}_k(S^i_{k+1},a')\\
&-\sum_{s'} p^{i}_{\mathcal{E}}(s'\mid S^i_k,A^i_k)\max_{a'}\vec Q^{i,z}_k(s',a')\Big).
\end{align*}
Since $A^i_k\in\{0,1\}$, exactly one of $(1-A^i_k)$ or $A^i_k$ is $1$; thus
\begin{align*}
|U_k|&\;\le\;\big|\vec r^{\,i}_{E_k}(S^i_k,j)-r^{i}_{\mathcal{E}}(S^i_k,j)\big| \\
&\;\le\;|\vec r^{\,i}_{E_k}(S^i_k,j)|+| r^{\,i}_{\mathcal{E}}(S^i_k,j)| \;\le\; 2R_{\max}.
\end{align*}
Moreover, by the fact that $\sum_{s'}p_{\mathcal{E}}^{\,i}(s'\mid S^i_k,A^i_k)=1$,
\begin{align*}
&|V_k|\;\le\;\gamma\Big(\big|\max_{a'}\vec Q^{i,z}_k(S^i_{k+1},a')\big| \\
&+\sum_{s'} p^{\,i}_\mathcal{E}(s'\mid S^i_k,A^i_k)\big|\max_{a'}\vec Q^{i,z}_k(s',a')\big|\Big)
\;\le\; 2\gamma \|\vec Q^{i,z}_k\|_{\infty}.
\end{align*}
Hence, using $(u+v)^2\le 2(u^2+v^2)$,
\[
\big|M^{i,z}_{k+1}\big|^2
\;\le\; 2\big(U_k^2+V_k^2\big)
\;\le\; 2\big(2R^2_{\max}+2\gamma \|\vec Q^{i,z}_k\|^2_{\infty}\big)
\]
So there is a constant $C$ such that
\[
\mathbb E\!\left[\big\|M^{i,z}_{k+1}\big\|^2\ \middle|\ \mathcal F_k\right]\;\le\; C\Bigl(1+\|\vec Q^{i,z}_k\|_\infty^2+\|\lambda^{\,i}_k\|_\infty^2\Bigr)\quad a.s.
\]
Since $M_{k+1}^{i,z}$ does not depend on $\lambda_k^i$, there exists $C<\infty$ such that
$\mathbb{E}[\|M_{k+1}^{i,z}\|^2\mid\mathcal F_k]\le C\bigl(1+\|\vec Q_k^{i,z}\|^2\bigr)
\le C\bigl(1+\|\vec Q_k^{i,z}\|^2+\|\lambda_k^i\|^2\bigr)$.

Finally, the slow‐timescale noise sequence $\{N_{k+1}\}$ is identically zero, so it trivially satisfies
\(\mathbb E[N_{k+1}\mid\mathcal{F}_k]=0\) and 
\(\mathbb E[\|N_{k+1}\|^2\mid\mathcal{F}_k]=0\le C(1+\|x_k\|_\infty^2+\|y_k\|_\infty^2)\).

Combining these observations establishes that both $\{M_{k+1}\}$ and $\{N_{k+1}\}$ are martingale–difference sequences with bounded conditional second moments, hence Assumption \ref{A3} holds. Finally, if we set $\frac{b_k}{a_k}\to 0$, Assumption \ref{A4} is satisfied.


We first prove the global asymptotic stability of the \emph{fast} ODE (Assumption~B1), then of the \emph{slow} ODE (Assumption~B2).

\paragraph{Fast–ODE attractor}
Rewrite the index‐update in Eq.~\eqref{Equ:QWI-index-update} as
\begin{equation}\label{eq:lambda-rec}
\lambda^i_{k+1}(z)=\lambda^i_k(z)+\alpha_k\,\frac{\beta_k}{\alpha_k}\Bigl(\vec{Q}_k^{i,z}(z, 1)-\vec{Q}_k^{i,z}(z, 0)\Bigr)\,.
\end{equation}
Let $\tau(k)=\sum_{m=0}^k\alpha_m\,,$ for $k\ge0$, and define the continuous interpolation of the Q‐iterates on each interval $t \in [\tau(k),\tau(k+1))$ by
\begin{align*}
\vec{\hat Q}_t^{i,z}(s, a)&=\vec{Q}_{k}^{i,z}(s, a)+\frac{t-\tau(k)}{\tau(k+1)-\tau(k)}\bigl(\vec{Q}_{k+1}^{i,z}(s, a)\\
&-\vec{Q}_{k}^{i,z}(s, a)\bigr)\,, \quad \hat\lambda^i_t(z)=\lambda^i_k(z).
\end{align*}
Since $\beta_k/\alpha_k\to 0$ as $k\to\infty$, the pair $(\vec{\hat Q}_t^{i,z}(s,a),\hat\lambda_t^i(z))$ tracks the ODE $\vec{\dot Q}_t^{i,z}(s,a)=h\bigl(\vec{\hat Q}_t^{i,z}(s,a),\tilde\lambda^i(z)\bigr)$ and $\dot\lambda_t^i(z)=0$, where $\tilde\lambda^i$ is the constant limit of any convergent subsequence of ${\lambda_n^i}$. Fix any index vector $\lambda$. Since $0\le\gamma<1$, for any two functions $\vec Q_k^{i,z},\vec Q_{k'}^{i,z}$,
\begin{align*}
&\|T_{\lambda} \vec Q^{i,z}_k - T_{\lambda} \vec Q^{i,z}_{k'}\|_\infty= \\
&\gamma\bigl\|\sum_{s' \in \mathcal{E}}p_\mathcal{E}^i(s'|s,a)\bigl(\max_{a'}\vec Q^{i,z}_k(s',a') - \max_{a'}\vec Q^{i,z}_{k'}(s',a')\bigr)\bigr\|_\infty \\
&\leq \gamma\,
\bigl\|\sum_{s' \in \mathcal{E}}p_\mathcal{E}^i(s'|s,a)\max_{a'}\bigl(\vec Q^{i,z}_k(s',a') - \vec Q^{i,z}_{k'}(s',a')\bigr)\bigr\|_\infty \\
&\;\le\;
\gamma\,\|\vec Q^{i,z}_k-\vec Q^{i,z}_{k'}\|_\infty,
\end{align*}
Thus $T_\lambda$ is a $\gamma$-contraction in $|\cdot|_\infty$; by Banach’s fixed-point theorem \cite{goebel1990topics}, there is a unique fixed point $\bar{\vec Q}^{i,z}_{\lambda}$ with $\bar{\vec Q}^{i,z}_{\lambda}=T_{\lambda}\bar{\vec Q}^{i,z}_{\lambda}$, and the ODE $\vec{\dot Q}_t^{i,z}(s,a)=T_{\tilde \lambda}\vec{\hat Q}_t^{i,z}(s,a)-\vec{\hat Q}_t^{i,z}(s,a)$ converges globally to $\bar{\vec Q}^{i,z}_{\lambda}$, which it is the fixed point of Eq.~\eqref{eq:MARBLE-bellman-opt} for a specific $\lambda$.

Let \(\bar{\vec Q}^{i,z}_{\lambda}\) and \(\bar{\vec Q}^{i,z}_{\nu}\) be the unique fixed points of 
\(\,T_{\lambda}\) and \(T_{\nu}\), respectively.  Then
\begin{align*}
&\bar{\vec Q}^{i,z}_{\lambda}-\bar{\vec Q}^{i,z}_{\nu}
=T_{\lambda}\bar{\vec Q}^{i,z}_{\lambda}-T_{\nu}\bar{\vec Q}^{i,z}_{\nu}\\
&=\bigl[T_{\lambda}\bar{\vec Q}^{i,z}_{\lambda}-T_{\lambda}\bar{\vec Q}^{i,z}_{\nu}\bigr]
 \;+\;\bigl[T_{\lambda}\bar{\vec Q}^{i,z}_{\nu}-T_{\nu}\bar{\vec Q}^{i,z}_{\nu}\bigr].
\end{align*}
Taking sup-norms and using the \(\gamma\)–contraction property gives
\begin{align*}
&\|\bar{\vec Q}^{i,z}_{\lambda}-\bar{\vec Q}^{i,z}_{\nu}\|_\infty
\;\le\;
\gamma\,\|\bar{\vec Q}^{i,z}_{\lambda}-\bar{\vec Q}^{i,z}_{\nu}\|_\infty \\
&\;+\;\|T_{\lambda}\bar{\vec Q}^{i,z}_{\nu}-T_{\nu}\bar{\vec Q}^{i,z}_{\nu}\|_\infty.
\end{align*}
In the second term we have,
\begin{align*}
&\|T_{\lambda}\bar{\vec Q}^{i,z}_{\nu}-T_{\nu}\bar{\vec Q}^{i,z}_{\nu}\|_\infty
=\sup_{z,a}\bigl( (1-a)\,\bigl|\lambda^i(z)-\nu^i(z)\bigr| \bigr) \\
&\;\le\;\|\lambda^i-\nu^i\|_\infty.
\end{align*}
and therefore $\|\bar{\vec Q}^{i,z}_{\lambda}-\bar{\vec Q}^{i,z}_{\nu}\|_\infty
\le\frac{\|\lambda^i-\nu^i\|_\infty}{1-\gamma}$ that shows \(\lambda^i\mapsto \bar{\vec Q}^{i,z}_{\lambda}\) is Lipschitz with constant \(1/(1-\gamma)\). Therefore, \ref{B1} is established.

\paragraph{Slow-ODE Attractor} Fix a reference state $z\in\mathcal S$. For a scalar subsidy $\lambda$, define the action--gap at $z$ by $f_\lambda(z):=\bar{\vec Q}^{\,i,z}_{\lambda}(z,1)-\bar{\vec Q}^{\,i,z}_{\lambda}(z,0)$. Since $\bar{\vec Q}^{\,i,z}_{\lambda}$ is Lipschitz, $f_\lambda(z)$ is continuous in $\lambda$; moreover, increasing $\lambda$ favors passivity (due to the MAI in Assumption \ref{assump:indexability-MARBLE}), so $f_\lambda(z)$ is decreasing in $\lambda$. Therefore, we define the Whittle index $\bar\lambda_\star^{\,i}(z):=\inf\{\lambda:\; z\in\bar{\mathcal P}(\lambda)\}$. By continuity and monotonicity, it is possible that $f_\lambda(z)=0$. Consider the slow ODE $\dot y(z)=\bar g_{y}(z):=f_{\,y(z)}(z)$ and the Lyapunov function $V_z\bigl(y(z)\bigr)=\tfrac12\bigl(y(z)-\bar\lambda_\star^{i}(z)\bigr)^2$. Since $f_\lambda(z)$ is decreasing in $\lambda$ with a unique zero at $\bar\lambda_\star^{i}(z)$, so along trajectories $\dot V_z=(y(z)-\bar\lambda_\star^{i}(z))f_{y(z)}(z)<0$ for $y(z)\neq\bar\lambda_\star^{i}(z)$ and $\dot V_z=0$ for $y(z)=\bar\lambda_\star^{i}(z)$, so $V_z$ is decreasing; by LaSalle’s invariance principle \cite{cheng2008extension}, $y(z,t)\to\bar\lambda_\star^{i}(z)$, hence we set the Whittle index $y^\star(z):=\bar\lambda_\star^{i}(z)$. Thus, the slow dynamics are globally asymptotically stable, with the unique equilibrium at the reference state $z$ given by $\bar\lambda_\star^{i}(z)$.

\paragraph{Boundedness}
By Theorem~2.1 in \cite{borkar2000ode}, asymptotic stability of the associated ODE implies stability of the recursion.
In our case, the ODE $\dot{\vec Q}^{\,i,z}_t(s,a)$ is globally asymptotically stable. Therefore, the stochastic recursion
\eqref{Equ:MARBLE-QL-update_rule_reform} is stable and $\sup_k\|x_k\|_\infty \;<\;\infty.$

Fix a reference state $z\in\mathcal S$ and define the action gap $f(\lambda)\equiv f_\lambda(z)\;:=\;\bar{\vec Q}^{\,i,z}_{\lambda}(z,1)-\bar{\vec Q}^{\,i,z}_{\lambda}(z,0)$. By Assumption~\ref{assump:indexability-MARBLE}, $f$ is continuous and decreasing in $\lambda$, and it is zero at $\bar\lambda^{i}_\star(z)$.
Fix $\rho>0$ and the compact “target” interval $\mathcal K_z \;:=\; [\,\bar\lambda^{\,i}_\star(z)-\rho,\;\bar\lambda^{\,i}_{\star}(z)+\rho\,]$. 
By continuity and monotonicity, there exists \(\eta=\eta(\rho)>0\) such that \(\lambda\ge \bar\lambda^{\,i}_{\star}(z)+\rho \Rightarrow f(\lambda)\le -\eta\) and \(\lambda\le \bar\lambda^{\,i}_\star(z)-\rho \Rightarrow f(\lambda)\ge \eta\). Let \(g_k:=\vec Q_k^{\,i,z}(z,1)-\vec Q_k^{\,i,z}(z,0)\) be the slow drift in \(y_{k+1}=y_k+\beta_k g_k\). Since the fast recursion is stable and \(\beta_k/\alpha_k\to 0\), two–timescale tracking (\cite[Ch.~6]{borkar2008stochastic}) gives \(\varepsilon_k:=\|\vec Q_k^{\,i,z}-\bar{\vec Q}^{\,i,z}_{\,y_k}\|_\infty\to 0\) a.s., hence \(|g_k-f(y_k)|\le 2\varepsilon_k\to 0\) a.s. Therefore, there exists (random but finite) \(k_0\) such that, for all \(k\ge k_0\), \(y_k \ge \bar\lambda^{\,i}_{\star}(z)+\rho \Rightarrow g_k \le -\eta/2\) and \(y_k \le \bar\lambda^{\,i}_\star(z)-\rho \Rightarrow g_k \ge \eta/2\), i.e., the drift points inward outside \(\mathcal K_z\) with magnitude at least \((\eta/2)\). Consequently, if \(y_m>\bar\lambda^{\,i}_{\star}(z)+\rho\) for some \(m\ge k_0\) and \(d:=y_m-(\bar\lambda^{\,i}_{\star}(z)+\rho)>0\), then choosing \(n>m\) minimal with \(\sum_{k=m}^{n-1}\beta_k\ge 2d/\eta\) yields \(y_n \le y_m - (\eta/2)\sum_{k=m}^{n-1}\beta_k \le \bar\lambda^{\,i}_{\star}(z)+\rho\), so the iterate re–enters \(\mathcal K_z\) in finite time (the lower tail is symmetric). Thus every excursion outside \(\mathcal K_z\) (for \(k\ge k_0\)) returns in finite time.

Next, bound the one–step increments. Since $\sup_k\|\vec Q_k^{\,i,z}\|_\infty<\infty$, there exists $G<\infty$
with $|g_k|\le 2\|\vec Q_k^{\,i,z}\|_\infty\le G$, hence $|y_{k+1}-y_k|\;\le\;G\beta_k$. In particular, over any finite number of steps, the cumulative change of $y_k$ is finite.

To conclude boundedness, fix any $y^\dagger\in\mathcal Z_z$ (e.g.\ $y^\dagger=\bar\lambda^{\,i}_\star(z)$) and
consider the coercive Lyapunov function $V(y):=\tfrac12(y-y^\dagger)^2$. For $k\ge k_0$,
writing $g_k=f(y_k)+e_k$ with $|e_k|\le \eta/2$, we compute
\begin{align*}
&V(y_{k+1})-V(y_k)
= \beta_k (y_k-y^\dagger) g_k + \tfrac12 \beta_k^2 g_k^2 \\
&= \beta_k (y_k-y^\dagger) f(y_k) + \beta_k (y_k-y^\dagger) e_k + \tfrac12 \beta_k^2 g_k^2 \\
&\le \underbrace{\beta_k (y_k-y^\dagger) f(y_k)}_{\le 0\ \text{since $f$ is nonincreasing}}
\;+\; \tfrac{\eta}{2}\beta_k |y_k-y^\dagger|
\;+\; \tfrac12 G^2 \beta_k^2.
\end{align*}
Taking conditional expectations and using $|y_k-y^\dagger|=\sqrt{2V(y_k)}$, an application of the
Robbins--Siegmund almost-supermartingale lemma
\cite{robbins1971convergence,borkar2008stochastic} (since $\sum_k\beta_k^2<\infty$) yields that $V(y_k)$
converges a.s., hence $\sup_k V(y_k)<\infty$ a.s. Because $V$ is coercive, we obtain $\sup_k |y_k|\;<\;\infty
\quad\text{a.s.}$.

Consequently, both $\{x_k\}$ and $\{y_k\}$ are almost surely bounded, establishing
\textup{(B3)}. Together with \textup{(A1)}–\textup{(A4)} and \textup{(B1)}–\textup{(B2)}, the two–timescale
SA theorem (Theorem~\ref{thm:Borkar-2TSA}) applies and the claimed convergence follows.

\end{proof}

\section{Simulation Results}\label{sec:experiments}

\begin{figure}[htbp]
  \centering
  \subfigure[Average reward]{%
    \includegraphics[width=0.44\textwidth]{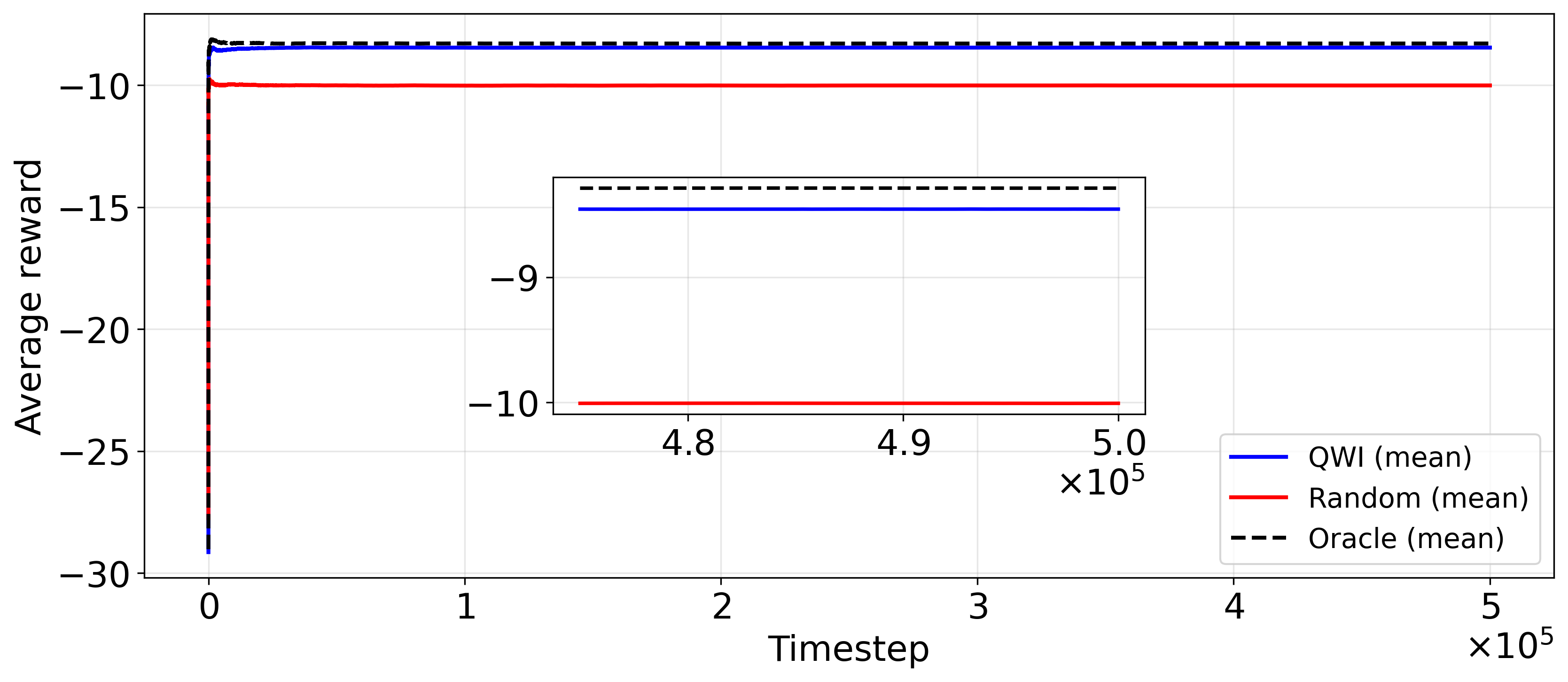}}
  \hfill
  \subfigure[Whittle index convergence]{%
    \includegraphics[width=0.44\textwidth]{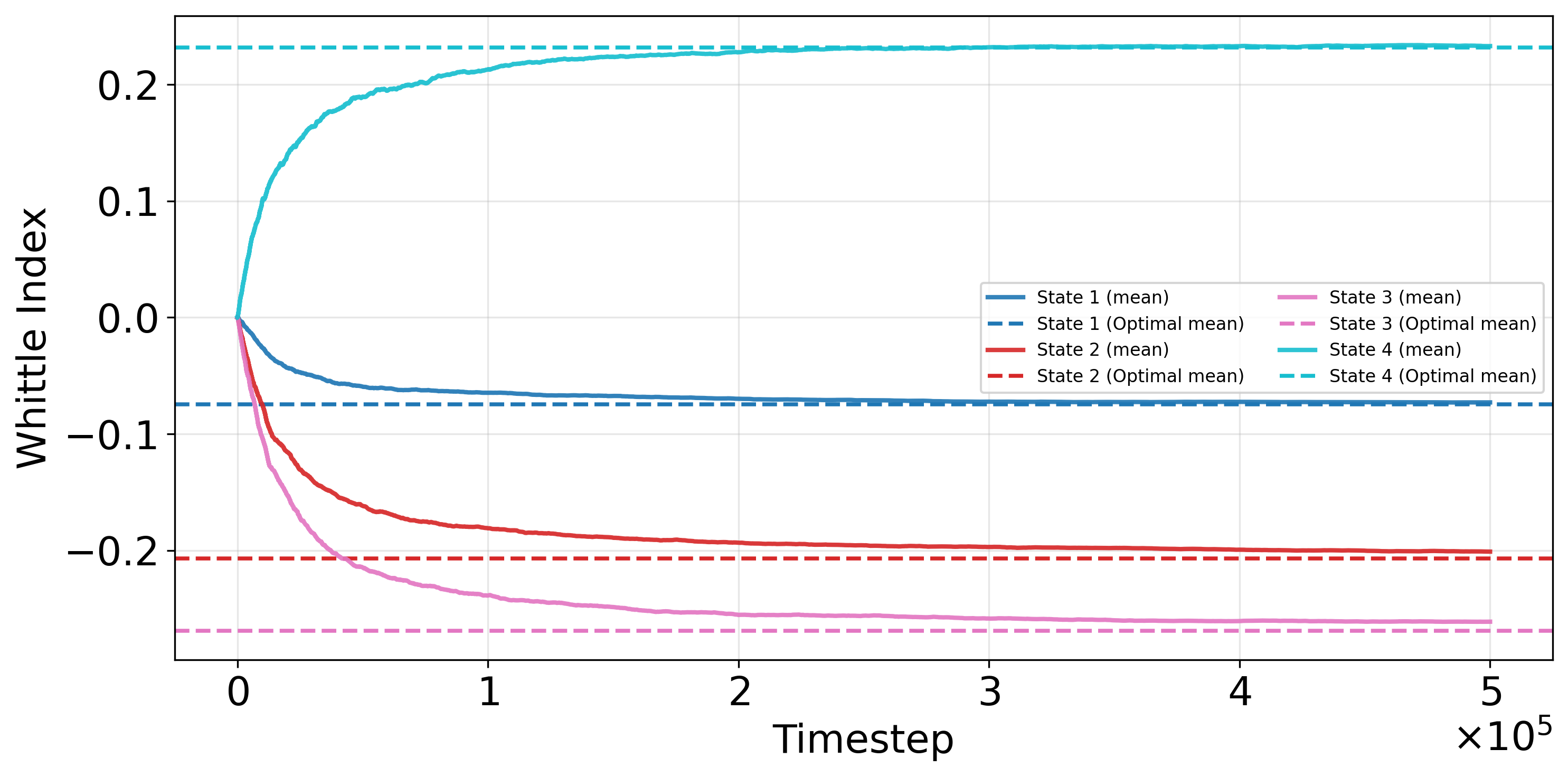}}
  \caption{Performance of synchronous QWI over $500{,}000$ iterations on the push-notification recommender task.}
  \label{fig:wide-results}
\end{figure}


We study the convergence of the synchronous QWI on the MARBLE framework on a mobile push-notification recommender system with a calibrated simulator where each user is an arm with a discrete engagement state \(s\in\{1,2,3,4\}\) (from low to high engagement). All of the implementation details are available at GitHub.\footnote{\url{https://github.com/cloud-commits/MARBLE-QWI}.}

Recommender systems personalize products, content, or services to increase engagement; on smartphones, this often occurs via push notifications, which are short, real-time messages delivered directly to the device \cite{resnick1997recommender}. Unlike pull-based recommendations, where users opt in by opening an app or visiting a site, pushes proactively interrupt, so targeting must balance engagement gains against annoyance and fatigue \cite{zhang2025robust}. 
At each decision step, the controller may activate \(M\) out of \(N\) users. Rewards capture immediate engagement (e.g., views and clicks) and depend on both the user state and an unobserved environment mode \(E_t\in\{E_1,E_2\}\) that switches according to a latent Markov chain. An unobserved latent state governs both rewards and transition dynamics, causing abrupt, exogenous nonstationarity. In other words, the latent state summarizes stochastic shocks affecting the recommender system and its users. Because arms evolve even when not pulled, the problem is restless.

We evaluate the synchronous QWI algorithm within the MARBLE framework in two configurations: homogeneous and heterogeneous. In the homogeneous case, all arms share the same transition probability and reward functions, but in the heterogeneous case, each arm has its own transition probability \(p^{\,i}_{e}(s' \mid s, a)\) and reward functions \(r^{\,i}_{e}(s,a)\). We simulate \(N=100\) users and let the recommender act on \(M=10\) users per time step. The step-size sequences are \(\alpha_k = 1/\lceil k/10000\rceil\) and \(\beta_k = \bigl(1+\lceil k \log k/10000\rceil\bigr)^{-1}\,\mathbb{I}\{k \bmod 10 = 0\}\), and we set \(\gamma = 0.8\) and \(\epsilon = 0.1\). The plots are the average of running the algorithm for 5 different seed numbers.

Fig~\ref{fig:wide-results}(a) reports the per-timestep reward averaged over users in the heterogeneous configuration. The results show that synchronous QWI converges to the oracle Whittle index policy, achieving nearly identical average rewards. Despite learning under environmental non-stationarity, the algorithm can efficiently approximate the optimal control. Fig~\ref{fig:wide-results}(b) illustrates the convergence of the learned Whittle indices for each state in the homogeneous configuration. As it is shown, over time, the learned indices gradually converge with the optimal ones, highlighting theoretical convergence guarantees. We also illustrated the performance of a random policy on selecting the Whittle indices as a baseline to compare it with synchronous QWI.

\section{Conclusion}
We introduced MARBLE, a restless bandit framework with a latent Markovian environment that induces nonstationarity, and analyzed synchronous QWI in this setting. Furthermore, we introduced the MAI criterion as a relaxed indexability assumption and proved that the synchronous QWI converges almost surely to the environment-averaged optimal Q-function and the corresponding Whittle indices under the MAI criterion, without requiring estimation of the latent state. Simulations of a calibrated simulator-embedded recommender system task corroborate the theory. Future work aims to target synchronous QWI with an erroneously calibrated simulator and asynchronous QWI.



\bibliographystyle{IEEEbib}
\bibliography{refs}

\end{document}